\documentclass[letterpaper]{article}
\usepackage{proceed2e}
\usepackage[margin=1in]{geometry}
\usepackage[utf8]{inputenc} % allow utf-8 input
\usepackage[T1]{fontenc}    % use 8-bit T1 fonts
\usepackage{hyperref}       % hyperlinks
\usepackage{url}            % simple URL typesetting
\usepackage{booktabs}       % professional-quality tables
\usepackage{amsfonts}       % blackboard math symbols
\usepackage{nicefrac}       % compact symbols for 1/2, etc.
\usepackage{microtype}      % microtypography
\usepackage{natbib}		% allows APA-like citations
\usepackage{graphicx}		% manages graphics in Latex
\usepackage{url}		% allows sensible printing of URLs
\usepackage{fancyhdr}		% allows headers and footers
\usepackage[table]{xcolor}
\usepackage{indentfirst}
\usepackage{diagbox}
\usepackage{graphicx}
\usepackage{subcaption}
\usepackage{caption}
\usepackage{amsmath}
\usepackage{amssymb}
\usepackage{amsthm}
\usepackage{thmtools,thm-restate}
\usepackage{paralist}
\usepackage{tabularx}
\usepackage{longtable}
\usepackage{multirow}
\usepackage{multicol}
\usepackage{fancyvrb}
\usepackage{framed}
\usepackage{comment}
\usepackage{algorithm}
\usepackage{algorithmic}
\usepackage{paralist}
\usepackage{listings}
\usepackage{multirow}
\usepackage{mathtools}
\usepackage{enumerate}
\usepackage{array}
\usepackage{times}

\theoremstyle{definition}
\newtheorem{definition}{Definition}[]

\newtheorem{lemma}{Lemma}

\newtheorem{corollary}{Corollary}

\newcommand{\RR}{\mathbb{R}}

\newcommand{\SPD}{\mathbb{S}}

\newcommand{\xx}{\mathbf{x}}

\DeclareMathOperator*{\argmin}{arg\,min}

\newcommand{\eps}{\varepsilon}
\newcommand{\defeq}{\vcentcolon=}
\newcommand{\eqdef}{=\vcentcolon}
\DeclarePairedDelimiterX{\inp}[2]{\langle}{\rangle}{#1, #2}

\DeclareMathOperator{\vect}{vec}

\newcommand{\tr}{\text{tr}}
\renewcommand{\d}[1]{\ensuremath{\operatorname{d}\!{#1}}}

\definecolor{dkgreen}{rgb}{0,0.6,0}
\definecolor{gray}{rgb}{0.5,0.5,0.5}
\definecolor{mauve}{rgb}{0.58,0,0.82}

\title{Frank-Wolfe Optimization for Symmetric-NMF under Simplicial Constraint}

\author{ 
\textbf{Han Zhao} \\
Machine Learning Department \\
Carnegie Mellon University\\
Pittsburgh, PA 15213 \\
\texttt{han.zhao@cs.cmu.edu}\\
\And
\textbf{Geoff Gordon} \\
Machine Learning Department \\
Carnegie Mellon University\\
Pittsburgh, PA 15213 \\
\texttt{ggordon@cs.cmu.edu}\\
}

\begin{document}

\maketitle

\begin{abstract}
Symmetric nonnegative matrix factorization has found abundant applications in various domains by providing a symmetric low-rank decomposition of nonnegative matrices. In this paper we propose a Frank-Wolfe (FW) solver to optimize the symmetric nonnegative matrix factorization problem under a simplicial constraint, which has recently been proposed for probabilistic clustering. Compared with existing solutions, this algorithm is simple to implement, and has no hyperparameters to be tuned. Building on the recent advances of FW algorithms in nonconvex optimization, we prove an $O(1/\eps^2)$ convergence rate to $\eps$-approximate KKT points, via a tight bound $\Theta(n^2)$ on the curvature constant, which matches the best known result in unconstrained nonconvex setting using gradient methods. Numerical results demonstrate the effectiveness of our algorithm. As a side contribution, we construct a simple nonsmooth convex problem where the FW algorithm fails to converge to the optimum. This result raises an interesting question about necessary conditions of the success of the FW algorithm on convex problems.
\end{abstract}

\section{INTRODUCTION}
\label{sec:intro}
Nonnegative matrix factorization (NMF) has found various applications in data mining, natural language processing, and computer vision~\citep{xu2003document,liu2003non,kuang2012symmetric}, due to its ability to provide low rank approximations and interpretable decompositions. The decision version of NMF is known to be NP-hard~\citep{vavasis2009complexity}, which also implies that its optimization problem is NP-hard. Recently, a variant of NMF where the input matrix is constrained to be symmetric has become popular for clustering~\citep{he2011symmetric,kuang2012symmetric,kuang2015symnmf}. The problem is known as symmetric NMF (SymNMF), and its goal is to minimize $||A - WW^T||_F^2$ under the constraint that $W \geq 0$ elementwise, where $A$ is a symmetric matrix of cluster affinities. Compared with NMF, SymNMF is applicable even when the algorithm does not have direct access to the data instances, but only their pairwise similarity scores. Note that in general the input matrix $A$ does not need to be nonnegative~\citep{kuang2012symmetric}. SymNMF has been successfully applied in many different settings and was shown to be competitive with standard clustering algorithms; see~\citep{kuang2012symmetric,kuang2015symnmf} and the references therein for more details. 

In this paper we investigate a constrained version of SymNMF where the input matrix is required to be both nonnegative and positive semidefinite. Furthermore, we require that $W$ is normalized such that each row of $W$ sums to 1. This problem has an interesting application in probabilistic clustering~\citep{zhao2015sof} where the $i$th row of $W$ can be interpreted as the probability that the $i$th data point lies in each clusters. Formally, we are interested in the following optimization problem, which we name as simplicial SymNMF (SSymNMF):
\begin{equation}
\begin{aligned}
& \underset{W}{\text{minimize}} && \frac{1}{4}||P - WW^T||_F^2 \\
& \text{subject to} && W \in \RR_+^{n\times k}, \quad W\mathbf{1}_k = \mathbf{1}_n
\end{aligned}
\label{equ:sof}
\end{equation}
where $P\geq 0$ and $P\in\SPD_+^n$ is positive semidefinite. (\ref{equ:sof}) was proposed as a formulation for probabilistic clustering~\citep{zhao2015sof}. The input matrix $P$ is interpreted as the co-cluster affinity matrix, i.e., entry $P_{ij}$ corresponds to the degree to which we encourage data instances $\mathbf{x}_i$ and $\mathbf{x}_j$ to be in the same cluster. Each row of $W$ then corresponds to the probability distribution of instance $\mathbf{x}_i$ being in different clusters. A similar simplicial constraint has been considered in NMF as well~\citep{nguyen2013simplicial}, where the goal is to seek a probabilistic part-based decomposition for clustering.

Previous approaches to solve SymNMF or SSymNMF use the penalty method to convert it into an unconstrained optimization problem and then solve it iteratively, using either first-order or second-order methods~\citep{kuang2012symmetric,zhao2015sof,kuang2015symnmf}. Such methods usually include two loops where the outer loop gradually increases the penalty coefficients and the inner loop finds a stationary point of each fixed penalized objective, hence they are often computationally expensive and slow to converge. In this paper we first give an equivalent geometric description of (\ref{equ:sof}) and then propose a variant of the classic Frank-Wolfe (FW) algorithm~\citep{frank1956algorithm}, a.k.a.\ the conditional gradient method~\citep{levitin1966constrained}, to solve it. We also provide a non-asymptotic convergence guarantee of our algorithm under an affine invariant stationarity measure (defined in Sec.~\ref{sec:preliminary}). More specifically, for a given approximation parameter $\eps > 0$, we show that the algorithm converges to an $\eps$-approximate KKT point of (\ref{equ:sof}) in $O(1/\eps^2)$ iterations. This rate is analogous to the one derived by~\citet{nesterov2013introductory} for general unconstrained problems (potentially nonconvex) using the gradient descent method, where the measure of stationarity is given by the norm of the gradient. The $O(1/\eps^2)$ rate has recently been shown to be optimal~\citep{cartis2010complexity} in the unconstrained setting for gradient methods, and it also matches the best known rate to a stationary point with (accelerated) projected gradient descent~\citep{ghadimi2016accelerated,ghadimi2016mini} in the constrained smooth nonconvex setting.

\textbf{Contributions}. We first give a generalized definition of the curvature constant~\citep{jaggi2013revisiting,lacoste2013block,lacoste2016convergence} that works for both convex and nonconvex functions, and we prove a tight bound of it in (\ref{equ:sof}). We then propose a convergence measure in terms of the duality gap and show that the gap is 0 iff KKT conditions are satisfied. Using these two tools, we propose a FW algorithm to solve (\ref{equ:sof}) and show that it has a non-asymptotic convergence rate $O(1/\eps^2)$ to KKT points. On the algorithmic side, we give a procedure that has the optimal linear time complexity and constant space complexity to implement the \emph{linear minimization oracle} (LMO) in the FW algorithm. As a side contribution, we construct a piecewise linear example where the FW algorithm fails to converge to the optimum. Surprisingly, we can also show that the FW algorithm works if we slightly change the objective function, despite that the new function remains piecewise linear and has an unbounded curvature constant. These two examples then raise an interesting question w.r.t. the necessary condition of the success of the FW algorithm. At the end, we conduct several numerical experiments to demonstrate the efficiency of the proposed algorithm by comparing it with the penalty method and projected gradient descent.

\section{PRELIMINARY}
\label{sec:preliminary}
\subsection{SymNMF UNDER SIMPLICIAL CONSTRAINT}
One way to understand (\ref{equ:sof}) is through its clustering based explanation: the goal is to find a probabilistic clustering of all the instances such that the given co-cluster affinity $P_{ij}$ for a pair of instances $(\mathbf{x}_i, \mathbf{x}_j)$ is close to the true probability that $\mathbf{x}_i$ and $\mathbf{x}_j$ reside in the same cluster:
\begin{align*}
\Pr(\mathbf{x}_i\sim\mathbf{x}_j) &= \sum_{h = 1}^k \Pr(c_i = c_j = h)\\
&= \sum_{h=1}^k \Pr(c_i = h)\Pr(c_j = h) = \mathbf{w}_{i}^T\mathbf{w}_j
\end{align*}
where we use the notation $\mathbf{x}_i\sim\mathbf{x}_j$ to mean ``$\mathbf{x}_i$ and $\mathbf{x}_j$ reside in the same cluster''; $c_i$ is the cluster assignment of $\mathbf{x}_i$ and $\mathbf{w}_i$ denotes the $i$th row vector of $W$. Note that the second equation holds because of the assumption of i.i.d. generation process of instances and their cluster assignments.

As a first note, the optimal solution $W^*$ to (\ref{equ:sof}) is not unique: for any permutation $\pi_n$ over $[n]$, an equivalent solution can be constructed by $W^*_{\pi_n} = W^*\Pi_{\pi_n}$, where $\Pi_{\pi_n}$ is a permutation matrix specified by $\pi_n$. This corresponds to an equivalence class of $W$ by label switching. Hence for any fixed $k$, there are at least $k!$ optimal solutions to (\ref{equ:sof}). The uniqueness of the solution to (\ref{equ:sof}) up to permutation is still an open problem. \citet{huang2014non} studied sufficient and necessary conditions for  the uniqueness of SymNMF, but they are NP-hard to check in general.

\citet{zhao2015sof} proposed a penalty method to transform (\ref{equ:sof}) into an unconstrained problem and solve it via sequential minimization. Roughly speaking, the penalty method repeatedly solves an unconstrained problem, and enforces the constraints in (\ref{equ:sof}) by gradually increasing the coefficients of the penalty terms. This process iterates until a solution is both feasible and a stopping criterion w.r.t. the objective function is met; see~\citep[Algo. 1]{zhao2015sof}. The penalty method contains 6 different hyperparameters to be tuned, and it is not even clear whether it will converge to a KKT point of (\ref{equ:sof}). To the best of our knowledge, no other methods has been proposed to solve (\ref{equ:sof}). To solve SymNMF, \citet{kuang2012symmetric} proposed a projected Newton method and \citet{vandaele2016efficient} developed a block coordinate descent method. However, due to the coupling of columns of $W$ introduced by the simplicial constraint, it is not clear how to extend these two algorithms to solve (\ref{equ:sof}). On the other hand, the simplicial constraint in (\ref{equ:sof}) restricts the feasible set to be compact, which makes it possible for us to apply the FW algorithm to solve it.

\subsection{Frank-Wolfe ALGORITHM}
The FW algorithm~\citep{frank1956algorithm,levitin1966constrained} is a popular first-order method to solve constrained convex optimization problems of the following form:
\begin{equation}
\text{minimize}_{\mathbf{x}\in\mathcal{D}}\quad f(\mathbf{x})
\label{equ:opt}
\end{equation}
where $f: \RR^d\to\RR$ is a convex and continuously differentiable function over the convex and compact domain $\mathcal{D}$. The FW method has recently attracted a surge of interest in machine learning due to the fact that it never requires us to project onto the constraint set $\mathcal{D}$ and its ability to cheaply exploit structure in $\mathcal{D}$~\citep{clarkson2010coresets,jaggi2011sparse,jaggi2013revisiting,lacoste2013block,lan2013complexity}. Compared with projected gradient descent, it provides arguably wider applicability since projection can often be computationally expensive (e.g., for the general $\ell_p$ ball), and in some case even computationally intractable~\citep{collins2008exponentiated}. At each iteration, the FW algorithm finds a feasible search corner $\mathbf{s}$ by minimizing a linear approximation at the current iterate $\mathbf{x}$ over $\mathcal{D}$:
\begin{equation}
w(\mathbf{x})\defeq \min_{\hat{\mathbf{s}}\in\mathcal{D}} f(x) + \nabla f(\mathbf{x})^T(\hat{\mathbf{s}} - \mathbf{x})
\label{equ:wolfe}
\end{equation}
The linear minimization oracle (LMO) at $\mathbf{x}$ is defined as:
\begin{equation}
\mathbf{s} = \text{LMO}(\mathbf{x}) \defeq \argmin_{\hat{\mathbf{s}}\in\mathcal{D}} f(x) + \nabla f(\mathbf{x})^T(\hat{\mathbf{s}} - \mathbf{x})
\label{equ:lmo}
\end{equation}
Given $\mathbf{s} = $ LMO$(\mathbf{x})$, the next iterate is then updated as a convex combination of $\mathbf{s}$ and the current iterate $\mathbf{x}$. If the FW algorithm starts with a corner as the initial point, then this property implies that at the $t$-th iteration, the current iterate is a convex combination of at most $t+1$ corners. The difference between $f(\mathbf{x})$ and $w(x)$ is known as the \emph{Frank-Wolfe gap} (FW-gap):
\begin{equation}
g(\mathbf{x}) \defeq f(\mathbf{x}) - w(\mathbf{x}) = \max_{\hat{\mathbf{s}}\in\mathcal{D}} \nabla f(\mathbf{x})^T(\mathbf{x} - \hat{\mathbf{s}})
\end{equation}
which turns out to be a special case of the general Fenchel duality gap when we transform (\ref{equ:opt}) into an unconstrained problem by adding an indicator function over $\mathcal{D}$ into the objective function~\citep[Appendix D]{lacoste2013block}. Due to the convexity of $f$, we have the following inequality: $\forall \mathbf{x}, \mathbf{y}\in\mathcal{D}, \quad w(\mathbf{y}) \leq f^* \leq f(\mathbf{x})$, where $f^*$ is the globally optimal value of $f$. Specifically, $\forall \mathbf{x}\in\mathcal{D}$, $w(\mathbf{x})\leq f(\mathbf{x})$ and as a result $g(\mathbf{x})$ can be used as an upper bound for the optimality gap: $\forall \mathbf{x}\in\mathcal{D}, \quad f(\mathbf{x}) - f^* \leq g(\mathbf{x})$, so that $g(\mathbf{x})$ is a certificate for the approximation quality of the current solution. Furthermore, the duality gap $g(\mathbf{x})$ can be computed \emph{essentially for free} in each iteration: $g(\mathbf{x}) = \nabla f(\mathbf{x})^T(\mathbf{x} - \mathbf{s})$.

It is well known that for smooth convex optimization problems the FW algorithm converges to an $\eps$-approximate solution in $O(1/\eps)$ iterations~\citep{frank1956algorithm,dunn1978conditional}. This result has recently been generalized to the setting where the LMO is solved only approximately~\citep{clarkson2010coresets,jaggi2013revisiting}. The analysis of convergence depends on a crucial concept known as the \emph{curvature constant} $C_f$ defined for a convex function $f$:
\begin{equation}
C_f \defeq \sup_{\substack{\mathbf{x}, \mathbf{s}\in \mathcal{D}, \\\gamma\in(0, 1], \\ \mathbf{y} = \mathbf{x} + \gamma(\mathbf{s} - \mathbf{x})}} \frac{2}{\gamma^2}\left(f(\mathbf{y}) - f(\mathbf{x}) - \nabla f(\mathbf{x})^T(\mathbf{y} - \mathbf{x})\right)
\label{equ:curvature}
\end{equation}
The curvature constant measures the relative deviation of a convex function $f$ from its linear approximation. It is clear that $C_f\geq 0$.  Furthermore, the definition of $C_f$ only depends on the inner product of the underlying space, which makes it affine invariant. In fact, the curvature constant, the FW algorithm, and its convergence analysis are all affine invariant~\citep{lacoste2013block}. Later we shall generalize the curvature constant to a function $f$ that is not necessarily convex and state our result in terms of the generalized definition. The above definition of the curvature constant still works for nonconvex functions, but for a concave function $f$, $C_f = 0$, which loses its geometric interpretation as measuring the curvature of $f$.

To proceed our discussion, we first establish some standard terminologies used in this paper. A continuously differentiable function $f$ is called $L$-smooth w.r.t. the norm $||\cdot||$ if $\nabla f$ is $L$-Lipschitz continuous w.r.t. the norm $||\cdot||$: $\forall \mathbf{x}, \mathbf{y}\in\mathcal{D}, \quad ||\nabla f(\mathbf{x}) - \nabla f(\mathbf{y})||\leq L||\mathbf{x} - \mathbf{y}||$. Throughout this paper, we will use $||\cdot||$ to mean the Euclidean norm, i.e., the $\ell_2$ norm for vectors and the Frobenius norm for matrices if not explicitly specified. It is standard to assume $f$ to be $L$-smooth in the convergence analysis of the FW algorithm~\citep{clarkson2010coresets,jaggi2013revisiting}. As we will see below, the smoothness of $f$ also implies the boundedness of $C_f$ on a compact set. 

\section{Frank-Wolfe FOR SymNMF UNDER SIMPLICIAL CONSTRAINT}
\label{sec:main}
\subsection{A GEOMETRIC PERSPECTIVE}
In this section we complement our discussion of (\ref{equ:sof}) in Sec.~\ref{sec:preliminary} with a geometric interpretation, which allows us to make a connection between the decision version of (\ref{equ:sof}) to the well known problem of completely positive matrix factorization~\citep{berman2003completely,dickinson2013copositive,dickinson2014computational}. The decision version of the optimization problem in (\ref{equ:sof}) is formulated as follows:
\begin{definition}[\textsf{SSymNMF}]
Given a matrix $P\in\RR_+^{n\times n}\cap \SPD_+^n$, can it be factorized as $P = WW^T$ for some integer $k$ such that $W\in\RR_+^{n\times k}$ and $W\mathbf{1}_k = \mathbf{1}_n$?
\end{definition}
Clearly, for $P\in\SPD_+^n$, we can decompose it as $P = UU^T$, where $U\in\RR^{n\times r}$, $r = \text{rank}(P)$. Let $\mathbf{u}_i$ be the $i$th row vector of $U$; then $P$ is the \emph{Gram matrix} of a set of $r$ dimensional vectors $\mathcal{U} = \{\mathbf{u}_1, \ldots, \mathbf{u}_n\}\subseteq\RR^r$. Similarly, $WW^T$ can be understood as the Gram matrix of a set of $k$ dimensional vectors $\mathcal{W} = \{\mathbf{w}_1, \ldots, \mathbf{w}_n\}\subseteq\RR^k$, where $\mathbf{w}_i$ is the $i$th row vector of $W$. The simplex constraint in (\ref{equ:sof}) further restricts each $\mathbf{w}_i\in\Delta^{k-1}$, i.e., $\mathbf{w}_i$ resides in the $k-1$ dimensional probability simplex. Hence equivalently, \textsf{SSymNMF} asks the following question:
\begin{definition}
\emph{Given a set of $n$ instances $\mathcal{U} = \{\mathbf{u}_i\}_{i=1}^n\subseteq\RR^r$, does there exist an integer $k$ and an embedding $\mathcal{T}: \RR^r\to\Delta^{k-1}$, such that inner product is preserved under $\mathcal{T}$, i.e., $\forall i,j\in[n], \inp{\mathbf{u}_i}{\mathbf{u}_j} = \inp{\mathcal{T}(\mathbf{u}_i)}{\mathcal{T}(\mathbf{u}_j)}$ ?}
\end{definition}

An affirmative answer to \textsf{SSymNMF} will give a certificate of the existence of such embedding $\mathcal{T}$: $\mathcal{T}(\mathbf{u}_i) = \mathbf{w}_i, \forall i\in[n]$. The goal of (\ref{equ:sof}) can thus be understood as follows: find an embedding into the probability simplex such that the discrepancy of inner products between the image space and the original space is minimized. The fact that inner product is preserved immediately implies that distances between every pair of instances are also preserved. If $k = r$, such an embedding is also known as an \emph{isometry}. In this case $\mathcal{T}$ is unitary. Note in the above definition of \textsf{SSymNMF} we do not restrict that $k = r$, and $\mathcal{T}$ does not have to be linear.

\textsf{SSymNMF} is closely connected to the strong membership problem for completely positive matrices~\citep{berman1994nonnegative,berman2003completely}. A completely positive matrix is a matrix $P$ that can be factorized as $P = WW^T$ where $W\in\RR_+^{n\times k}$ for some $k$. The set of completely positive matrices forms a convex cone, and is known as the completely positive cone (CP cone). From this definition we can see that the decision version of SymNMF corresponds to the strong membership problem of the CP cone, which has recently been shown by \citet{dickinson2014computational} to be NP-hard. Geometrically, the strong membership problem for CP matrices asks the following question:
\begin{definition}[\textsf{SMEMCP}]
\emph{Given a set of $n$ instances $\mathcal{U} = \{\mathbf{u}_i\}_{i=1}^n\subseteq\RR^r$, does there exist an integer $k$ and an embedding $\mathcal{T}: \RR^r\to\RR_+^k$, such that inner product is preserved under $\mathcal{T}$, i.e., $\forall i,j\in[n], \inp{\mathbf{u}_i}{\mathbf{u}_j} = \inp{\mathcal{T}(\mathbf{u}_i)}{\mathcal{T}(\mathbf{u}_j)}$ ?}
\end{definition}

Note that the only difference between \textsf{SSymNMF} and \textsf{SMEMCP} lies in the range of the image space: the former asks for an embedding in $\Delta^{k-1}$ while the latter only asks for embedding to reside in $\RR_+^k$. \textsf{SSymNMF} is therefore conjectured to be NP-hard as well, but this assertion has not been formally proved yet. A simple reduction from \textsf{SMEMCP} to \textsf{SSymNMF} does not work: a ``yes'' answer to the former does not imply a ``yes'' answer to the latter.

\subsection{ALGORITHM}
We list the pseudocode of the FW method in Alg.~\ref{alg:fw} and discuss how Alg.~\ref{alg:fw} can be efficiently applied and implemented to solve SSymNMF. We start by deriving the gradient of $f(W) = \frac{1}{4}||P - WW^T||_F^2$:
\begin{equation}
\nabla f(W) = (WW^T - P)W \in\RR^{n\times k}
\label{equ:grad}
\end{equation}
The normalization constraint keeps the feasible set decomposable (even though the objective function $f$ is not decomposable): it can be equivalently represented as the product of $n$ probability simplices $\Delta^{k-1}\times\cdots\times\Delta^{k-1}\eqdef \Pi_n\Delta^{k-1}$. Hence at the $t$-th iteration of the algorithm we solve the following linear optimization problem over $\Pi_n\Delta^{k-1}$:
\begin{equation}
\begin{aligned}
& \underset{S}{\text{minimize}} && \tr\left(S^T\nabla f(W^{(t)})\right) \\
& \text{subject to} && S\in \Pi_n\Delta^{k-1}
\end{aligned}
\label{equ:lmo}
\end{equation}

Because of the special structure of the constraint set in (\ref{equ:lmo}), given $\nabla f(W)$, we can efficiently compute the two key quantities $\mathbf{x}^{(t+1)}$ and $g_t$ in Alg.~\ref{alg:fw} in $O(nk)$ time and $O(1)$ space. The pseudocode is listed in Alg.~\ref{alg:lmo}. The key observation that allows us to achieve this efficient implementation is that at each iteration $t\in[T]$, LMO$(\mathbf{x}^{(t)})$ is guaranteed to be a sparse matrix that contains exactly one 1 in each row. The time complexity of Alg.~\ref{alg:lmo} is $3nk$, which can be further reduced to $2nk$ by additional $O(n)$ space to store the index of the nonzero element at each row of LMO$(\mathbf{x}^{(t)})$. Alg.~\ref{alg:fw}, together with Alg.~\ref{alg:lmo} as a sub-procedure, is very efficient while at the same time being simple to implement. Furthermore, it does not have any hyperparameter to be tuned: this is in sharp contrast with the penalty method.

\begin{algorithm}[htb]
\centering
\caption{Frank-Wolfe algorithm (non-convex variant)}
\label{alg:fw}
\begin{algorithmic}[1]
\REQUIRE        Initial point $\mathbf{x}^{(0)}\in\mathcal{D}$, approximation parameter $\eps > 0$
\FOR        {$t = 0$ to $T$}
    \STATE  Compute $\mathbf{s}^{(t)} = \text{LMO}(\mathbf{x}^{(t)})\defeq \argmin_{\mathbf{s}\in \mathcal{D}}\mathbf{s}^T\nabla f(\mathbf{x}^{(t)})$
    \STATE  Compute update direction $\mathbf{d}_t\defeq \mathbf{s}^{(t)} - \mathbf{x}^{(t)}$
    \STATE  Compute FW-gap $g_t\defeq -\nabla f(\mathbf{x}^{(t)})^T\mathbf{d}_t$
    \STATE  \textbf{if} $g_t\leq\eps$ \textbf{then return} $\mathbf{x}^{(t)}$
    \STATE  Compute $\gamma_t \defeq \min\{g_t/C, 1\}$ for any $C\geq \bar{C}_f$ (defined in (\ref{equ:newcurvature}))
    \STATE  Update $\mathbf{x}^{(t+1)}\defeq \mathbf{x}^{(t)} + \gamma_t \mathbf{d}_t$
\ENDFOR
\RETURN $\mathbf{x}^{(T)}$
\end{algorithmic}
\end{algorithm}
\begin{algorithm}[htb]
\centering
\caption{Compute next iterate and the gap function}
\label{alg:lmo}
\begin{algorithmic}[1]
\REQUIRE    $\nabla f(\mathbf{x}^{(t)})$, $\mathbf{x}^{(t)}$
\STATE  $g_t\defeq \inp{\nabla f(\mathbf{x}^{(t)})}{\mathbf{x}^{(t)}}$
\FOR {$i = 1$ to $n$}
    \STATE  $g_t \leftarrow g_t - \min_{j\in[k]}\nabla f(\mathbf{x}^{(t)})_{ij}$
\ENDFOR
\STATE  Compute $\gamma_t \defeq \min\{g_t/C, 1\}$
\STATE  $\mathbf{x}^{(t+1)}\defeq (1-\gamma_t)\mathbf{x}^{(t)}$
\FOR {$i = 1$ to $n$}
    \STATE  $j_i \defeq \argmin_{j\in[k]}\nabla f(\mathbf{x}^{(t)})_{ij}$
    \STATE  $\mathbf{x}^{(t+1)}_{ij_i}\leftarrow \mathbf{x}^{(t+1)}_{ij_i} + \gamma_t$
\ENDFOR
\RETURN     $\mathbf{x}^{(t+1)}, g_t$
\end{algorithmic}
\end{algorithm}

\subsection{CONVERGENCE ANALYSIS}
In this section we provide a non-asymptotic convergence rate for the FW algorithm for solving (\ref{equ:sof}) and derive a tight bound for its curvature constant. Our analysis is based on the recent work~\citep{lacoste2016convergence}, where we redefine the curvature constant so that the new definition works in both convex and nonconvex settings. Due to space limit, we only provide partial proofs for theorems and lemmas derived in this paper, and refer readers to supplementary material for all detailed proofs. 

When applied to smooth convex constrained optimization problems, the FW algorithm is known to converge to an $\eps$-approximate solution in $O(1/\eps)$ iterations. However the convergence of global optimality is usually unrealistic to hope for in nonconvex optimization, where even checking local optimality itself can be computationally intractable~\citep{murty1987some}. Clearly, to talk about convergence, we need to first establish a convergence criterion. For unconstrained nonconvex problems, the norm of the gradient $||\nabla f||$ has been used to measure convergence~\citep{ghadimi2016mini,ghadimi2016accelerated}, since $\lim_{t\rightarrow\infty}||\nabla f(\mathbf{x}^{(t)})|| = 0$ means every limit point of the sequence $\{\mathbf{x}^{(t)}\}$ is a stationary point. But such a convergence criterion is not appropriate for constrained problems because a stationary point can lie on the boundary of the feasible region while not having a zero gradient. To address this issue, we will use the FW-gap $g(\mathbf{x})$ as a measure of convergence. Note that the gap $g(\mathbf{x})$ works as an optimality gap only if the original problem is convex. To see why this is also a good measure of convergence for constrained nonconvex problems, we first prove the following theorem:
\begin{restatable}{theorem}{kkt}
\label{thm:converge}
Let $f$ be a differentiable function and $\mathcal{D}$ be a convex compact domain. Define $g(\mathbf{x}) \defeq \max_{\hat{\mathbf{s}}\in\mathcal{D}} \nabla f(\mathbf{x})^T(\mathbf{x} - \hat{\mathbf{s}})$. Then $\forall \mathbf{x}\in\mathcal{D}, g(\mathbf{x}) \geq 0$ and $g(\mathbf{x}) = 0$ iff $\mathbf{x}$ is a Karush–Kuhn–Tucker (KKT) point.
\end{restatable}
\begin{proof}
To see $g(\mathbf{x}) \geq 0$, we have:
$$g(\mathbf{x}) \defeq \max_{\hat{\mathbf{s}}\in\mathcal{D}} \nabla f(\mathbf{x})^T(\mathbf{x} - \hat{\mathbf{s}}) \geq \nabla f(\mathbf{x})^T(\mathbf{x} - \mathbf{x}) = 0$$
Reformulate the original constrained problem in the following way:
\begin{align}
\label{equ:reformu}
& \underset{\mathbf{x}}{\text{minimize}} \quad f(\mathbf{x})\nonumber\\
& \text{subject to} \quad \mathbb{I}_{\mathcal{D}}(\mathbf{x}) \leq 0
\end{align}
where $\mathbb{I}_{\mathcal{D}}(\mathbf{x})$ is the indicator function of $\mathcal{D}$ which takes value 0 iff $\mathbf{x}\in\mathcal{D}$ otherwise $\infty$. Define $\mathcal{N}_\mathcal{D}(\mathbf{x})$ as the \emph{normal cone} at $\mathbf{x}$ in a convex set $\mathcal{D}$:
\begin{equation*}
\mathcal{N}_\mathcal{D}(\mathbf{x})\defeq \{\mathbf{z}\mid \mathbf{z}^T\mathbf{x}\geq \mathbf{z}^T\mathbf{y}, \forall \mathbf{y}\in\mathcal{D}\}
\end{equation*}
and realize the fact that the subdifferential of the indicator function when $\xx\in\mathcal{D}$ is precisely the normal cone at $\xx$, i.e., $
\partial~\mathbb{I}_\mathcal{D}(\mathbf{x}) = \mathcal{N}_{\mathcal{D}}(\mathbf{x})$, we have:
\begin{align*}
& g(\mathbf{x}) = 0 ~\Leftrightarrow~ \max_{\mathbf{y}\in\mathcal{D}} \nabla f(\mathbf{x})^T(\mathbf{x} - \mathbf{y}) = 0\\
\Leftrightarrow~ &\forall \mathbf{y}\in\mathcal{D}, \nabla f(\mathbf{x})^T(\mathbf{y} - \mathbf{x})\geq 0 ~\Leftrightarrow~ -\nabla f(\xx)\in \mathcal{N}_{\mathcal{D}}(\mathbf{x}) 
\end{align*}
Note that the normal cone for a convex set is a convex cone, which implies that $\exists \lambda \geq 0$, s.t., 
\begin{equation*}
\nabla f(\xx) + \lambda\partial~\mathbb{I}_\mathcal{D}(\mathbf{x}) = 0~\Leftrightarrow~\nabla_{\xx}\mathcal{L}(x, \lambda) = 0
\end{equation*}
where $\mathcal{L}(\xx, \lambda)\defeq f(\xx) + \lambda~\mathbb{I}_\mathcal{D}(\mathbf{x})$ is the Lagrangian of \eqref{equ:reformu}. By construction, $\lambda \geq 0$ satisfies the dual feasibility condition and $\xx\in\mathcal{D}$ satisfies the primal feasibility condition, which also means the complementary slackness is satisfied, i.e., $\lambda~\mathbb{I}_\mathcal{D}(\mathbf{x}) = 0$. It follows that $g(\xx) = 0$ iff $\xx$ is a KKT point of \eqref{equ:reformu}.
\end{proof}
\textbf{Remark}. Note that in Thm.~\ref{thm:converge} we do not assume $f$ to be convex. In fact we can also relax the differentiability of $f$, as long as the subgradient exists at every point of $f$. The proof relies on the fact that $g(\mathbf{x}) = 0$ implies $-\nabla f(\mathbf{x})$ is in the normal cone at $\mathbf{x}$. Thm.~\ref{thm:converge} justifies the use of $g(\mathbf{x})$ as a convergence measure: if $\lim_{t\rightarrow\infty} g(\mathbf{x}^{(t)}) = 0$, then by the continuity of $g$, every limit point of $\{\mathbf{x}^{(k)}\}$ is a KKT point of $f$. Since being a KKT point is also a sufficient condition for optimality when $f$ is convex, Thm.~\ref{thm:converge} also recovers the case where $g(\xx)$ is used as convergence measure for convex problems. 

For a continuously differentiable function $f$, we now extend the definition of curvature constant as follows:
\begin{equation}
\bar{C}_f \defeq \sup_{\substack{\mathbf{x}, \mathbf{s}\in \mathcal{D}, \\\gamma\in(0, 1], \\ \mathbf{y} = \mathbf{x} + \gamma(\mathbf{s} - \mathbf{x})}} \frac{2}{\gamma^2}\Bigl\lvert f(\mathbf{y}) - f(\mathbf{x}) - \nabla f(\mathbf{x})^T(\mathbf{y} - \mathbf{x})\Bigr\rvert
\label{equ:newcurvature}
\end{equation}
Clearly $\bar{C}_f\geq 0$ and the new definition reduces to $C_f$ when $f$ is a convex function. In general we have $\bar{C}_f\geq C_f$ for any function $f$. Again, $\bar{C}_f$ measures the relative deviation of $f$ from its linear approximation, and is still affine invariant. The difference of $\bar{C}_f$ from the original $C_f$ becomes clear when $f$ is concave: in this case $C_f = 0$, but $\bar{C}_f > 0$ and $\bar{C}_f = C_{-f}$. On the downside, the fact that $\bar{C}_f\geq C_f$ means our asymptotic bound is of constant times larger than the original one proved by~\citet{lacoste2016convergence}, but they share the same asymptotic rate in terms of the given precision parameter $\eps$. Finally, $\bar{C}_f = 0$ iff $f$ is affine. As in \citep{jaggi2013revisiting,lacoste2013block}, for a smooth function $f$ with Lipschitz constant $L$ over compact set $\mathcal{D}$, we can bound $\bar{C}_f$ in terms of $L$:
\begin{restatable}{lemma}{curvature}
Let $f$ be a $L$-smooth function over a convex compact domain $\mathcal{D}$, and define $\text{diam}(\mathcal{D})\defeq \sup_{\mathbf{x},\mathbf{y}\in\mathcal{D}}||\mathbf{x} - \mathbf{y}||$. Then $\bar{C}_f\leq \text{diam}^2(\mathcal{D})L$.
\label{thm:curvature}
\end{restatable}
The proof of Lemma~\ref{thm:curvature} does not require $f$ to be convex. Furthermore, $f$ does not need to be second-order differentiable --- being smooth is sufficient. We proceed to derive a convergence bound for Alg.~\ref{alg:fw} using our new $\bar{C}_f$, which better reflects the geometric nonlinearity of nonconvex functions. The main idea of the proof is to bound the decrease of the gap function by minimizing a quadratic function iteratively. 
\begin{restatable}{theorem}{lacoste}
\label{thm:nonconvex}
Consider the problem (\ref{equ:opt}) where $f$ is a continuously differentiable function that is potentially nonconvex, but has a finite curvature constant $\bar{C}_f$ as defined by (\ref{equ:newcurvature}) over the compact convex domain $\mathcal{D}$. Consider running Frank-Wolfe (Algo.~\ref{alg:fw}), then the minimal FW gap $\tilde{g}_T\defeq \min_{0\leq t\leq T}g_t$ encountered by the iterates during the algorithm after $T$ iterations satisfies:
\begin{equation}
\tilde{g}_T \leq \frac{\max\{2h_0\bar{C}_f, \sqrt{2h_0\bar{C}_f}\}}{\sqrt{T+1}},\quad \forall T \geq 0
\end{equation}
where $h_0\defeq f(\mathbf{x}^{(0)}) - \min_{\mathbf{x}\in\mathcal{D}}f(\mathbf{x})$ is the initial global suboptimality. It thus takes at most $O(1/\eps^2)$ iterations to find an approximate KKT point with gap smaller than $\eps$.
\end{restatable}

We comment that the $O(1/\varepsilon^2)$ convergence rate is the same as the one provided by~\citet{lacoste2016convergence} using the original curvature constant $C_f$ for smooth nonconvex functions, and it is also analogous to the ones derived for gradient descent for unconstrained smooth problems and (accelerated) projected gradient descent for constrained smooth problems. The convergence rate of Alg.~\ref{alg:fw} depends on the curvature constant of $f$ over $\mathcal{D}$. We now bound the smoothness constant $L$ and the diameter of the feasible set in (\ref{equ:sof}).

\textbf{Bound on smoothness constant}.~The objective function in (\ref{equ:sof}) is second-order differentiable, hence we can bound the smoothness constant by bounding the spectral norm of the Hessian instead:
\begin{lemma}[\citet{nesterov2013introductory}]
\label{lemma:lipschitz}
Let $f$ be twice differentiable. If $||\nabla^2 f(\mathbf{x})||_2 \leq L$ for all $\mathbf{x}$ in the domain, then $f$ is $L$-smooth.
\end{lemma}

Note that (\ref{equ:grad}) is a matrix function of a matrix variable, whose Hessian is a matrix of order $nk\times nk$. Although one can compute all the elements of the Hessian by computing all the partial derivatives $\partial \nabla f(W)_{ij}/\partial W_{st}$ separately, this approach is tedious and may hide the structure of the Hessian matrix. Instead we apply matrix differential calculus~\citep{magnus1985matrix,magnus2010concept} to derive the Hessian:
\begin{restatable}{lemma}{hessian}
Let $f(W) = \frac{1}{4}||P - WW^T||_F^2$ and define $\nabla^2 f(W) \defeq \partial \vect{\nabla f(W)} / \partial \vect{W}$. Then:
\begin{align}
\nabla^2 f(W) &= W^TW\otimes I_n + I_k\otimes (WW^T - P) \nonumber\\
& + (W^T\otimes W)K_{nk}
\label{equ:hessian}
\end{align}
where $K_{nk}$ is a commutation matrix such that $K_{nk}\vect{W} = \vect{W^T}$.
\end{restatable}
The derivation of the above lemma uses the matrix differential calculus with basic properties of tensors and commutation matrices. Before we proceed, we first present a lemma that will be useful to bound the spectral norm of the above Hessian matrix:
\begin{restatable}{lemma}{eigenbound}
$\sup_{\substack{W \geq 0, \\ W\mathbf{1}_k = \mathbf{1}_n}} ||W^TW||_2 = n$.
\label{thm:eigenbound}
\end{restatable}
We are now ready to bound the spectral norm of the Hessian $\nabla^2 f(W)$ and use it to bound the smoothness constant of $f$.
\begin{restatable}{lemma}{smooth}
\label{lemma:smooth}
Let $c \defeq ||P||_2$. $f = \frac{1}{4}||P - WW^T||_F^2$ is $(3n + c)$-smooth on $\mathcal{D} = \{W \in\RR_+^{n\times k}\mid W\mathbf{1}_k = \mathbf{1}_n\}$.
\end{restatable}
The above lemma follows from Lemma~\ref{lemma:lipschitz} where we use Lemma~\ref{thm:eigenbound} to help bound the spectral norm of the Hessian matrix in \eqref{equ:hessian}.

\textbf{Bound on diameter of $\mathcal{D}$}.~The following lemma can be easily shown:
\begin{restatable}{lemma}{diameter}
Let $\mathcal{D} = \{W \in\RR_+^{n\times k}\mid W\mathbf{1}_k= \mathbf{1}_n\}$. Then $\text{diam}^2(\mathcal{D}) = 2n$ with respect to the Frobenius norm.
\label{lemma:diameter}
\end{restatable}
Combining Lemma~\ref{lemma:diameter} with Lemma~\ref{lemma:smooth} and assuming $c$ is a constant that does not depend on $n$, we immediately have $\bar{C}_f \leq 2n(3n + c) = O(n^2)$ by Lemma~\ref{thm:curvature}. A natural question to ask is: can we get better dependency on $n$ in the upper bound for $\bar{C}_f$ given the special structure that $\mathcal{D} = \Pi_n\Delta^{k-1}$? The answer is negative, as we can prove the following lower bound on the Hessian:
\begin{restatable}{lemma}{bound}
$\inf_{\substack{W \geq 0, \\ W\mathbf{1}_k = \mathbf{1}_n}}||\nabla^2 f(W)||_2 \geq n/k^2 - c$.
\label{lemma:bound}
\end{restatable}
Using Lemma~\ref{lemma:bound}, we can prove a tight bound on our curvature constant $\bar{C}_f$:
\begin{restatable}{theorem}{tight}
\label{thm:tightbound}
The curvature constant $\bar{C}_f$ for $f = \frac{1}{4}||P - WW^T||_F^2$ on $\mathcal{D} = \{W\in\RR_+^{n\times k}\mid W\mathbf{1}_k = \mathbf{1}_n\}$ satisfies:
$$2n(n/k^2 - c) \leq \bar{C}_f\leq 2n(3n + c)$$
where $c\defeq ||P||_2$. Specifically, we have $\bar{C}_f = \Theta(n^2)$.
\end{restatable}
\begin{proof}
The upper bound part is clear by combining Lemma~\ref{thm:curvature} and Lemma~\ref{lemma:smooth}. We only need to show the lower bound. Since $f$ is twice-differentiable, we have:
\begin{align*}
\bar{C}_f &\defeq \sup_{\substack{\mathbf{x}, \mathbf{s}\in \mathcal{D}, \\\gamma\in(0, 1], \\ \mathbf{y} = \mathbf{x} + \gamma(\mathbf{s} - \mathbf{x})}} \frac{2}{\gamma^2}\Bigl\lvert f(\mathbf{y}) - f(\mathbf{x}) - \nabla f(\mathbf{x})^T(\mathbf{y} - \mathbf{x})\Bigr\rvert \\
&= \sup_{\substack{\mathbf{x}, \mathbf{s}\in \mathcal{D}, \\\gamma\in(0, 1], \\ \mathbf{y} = \mathbf{x} + \gamma(\mathbf{s} - \mathbf{x})}} \frac{2}{\gamma^2}\frac{1}{2}|(\mathbf{y} - \mathbf{x})^T \nabla^2 f(\xi)(\mathbf{y} - \mathbf{x})| \\
&\geq \sup_{\substack{\mathbf{x}, \mathbf{s}\in \mathcal{D}, \\\gamma\in(0, 1], \\ \mathbf{y} = \mathbf{x} + \gamma(\mathbf{s} - \mathbf{x})}}\frac{1}{\gamma^2} ||\mathbf{y} - \mathbf{x}||_2^2~\cdot  \inf_{\xi\in\mathcal{D}}||\nabla^2 f(\xi)||_2 \\
&\geq \frac{1}{\gamma^2}\gamma^2 \text{diam}^2(\mathcal{D})\left(\frac{n}{k^2} - c\right) = 2n(\frac{n}{k^2} - c) 
\end{align*}
The second equality is due to the mean-value theorem, and the third inequality holds by the definition of $\inf$. The last inequality  follows from Lemma~\ref{lemma:diameter} and \ref{lemma:bound}. 
\end{proof}

Combining all the analysis above and using Alg.~\ref{alg:lmo} to implement the linear minimization oracle in Alg.~\ref{alg:fw}, we can bound the time complexity of the FW algorithm to solve (\ref{equ:sof}):
\begin{corollary}
The FW algorithm (Alg.~\ref{alg:fw}) achieves an $\eps$-approximate KKT point of (\ref{equ:sof}) in $O(n^3k/\eps^2)$ time.
\end{corollary}
\begin{proof}
In each iteration Alg.~\ref{alg:fw} takes $O(nk)$ time to compute the gap function as well as the next iterate. Based on Thm.~\ref{thm:nonconvex}, the iteration complexity to achieve an $\eps$-approximate KKT point is $O(\bar{C}_f/\eps^2)$. The result follows from Thm.~\ref{thm:tightbound} showing that $\bar{C}_f = \Theta(n^2)$.
\end{proof}

\section{A FAILURE CASE OF Frank-Wolfe ON A NONSMOOTH CONVEX PROBLEM}
For convex problems, the theoretical convergence guarantee of Frank-Wolfe algorithm and its variants depends crucially on the smoothness of the objective function: \citep{freund2016new,garber2015faster,harchaoui2015conditional,nesterov2015complexity} require the gradient of the objective function to be Lipschitz continuous or H\"older continuous; the analysis in \citep{jaggi2013revisiting} requires a finite curvature constant $C_f$. The Lipschitz continuous gradient condition is sufficient for the convergence analysis, but not necessary: \citet{odor2016frank} shows that Frank-Wolfe works for a Poisson phase retrieval problem, where the gradient of the objective is not Lipschitz continuous. As we discuss in the last section, Lipschitz continuous gradient implies a finite curvature constant. Hence an interesting question to ask is: does there exist a constrained convex problem with unbounded curvature constant such that the Frank-Wolfe algorithm can find its optimal solution? Surprisingly, the answer to the above question is affirmative. But before we give the example, we first construct an example where the FW algorithm fails when the objective function is convex but nonsmooth.

We first construct a very simple example where the objective function is a piecewise linear and the constraint set is a polytope. We will show that when applied to this simple problem, the FW algorithm, along with its line search variant and its fully corrective variant, do not even converge to the optimal solution. In fact, as we will see shortly, the limit point of the sequence can be arbitrarily far away from the global optimum. Consider the following convex, constrained optimization problem:
\begin{align}
\underset{x_1, x_2}{\text{minimize}} & \quad\max\{5x_1 + x_2, -5x_1 + x_2\} \\
\text{subject to} & \quad x_2 \leq 3, 3x_1 + x_2 \geq 0, -3x_1 + x_2 \geq 0\nonumber
\label{equ:example}
\end{align}
We plot the objective function of this example in the left figure of Fig.~\ref{fig:example}. The unique global optimum for this problem is given by $\mathbf{x}^* = (0, 0)$ with $f(\mathbf{x}^*) = 0$. The feasible set $\mathcal{D}$ contains three vertices: $(-1, 3)$, $(1, 3)$ and $(0, 0)$. If we apply the FW algorithm to this problem, it is straightforward to verify that the $\text{LMO}(\mathbf{x})$ is given by:
$$
\text{LMO}(\mathbf{x}) =
\begin{cases}
(-1, 3) & x_1 > 0 \\
(1, 3) & x_1 < 0 \\
\{(-1, 3), (1, 3), (0, 0)\} & x_1 = 0
\end{cases}
$$
\begin{figure}[htb]
\centering
	\includegraphics[width=\linewidth]{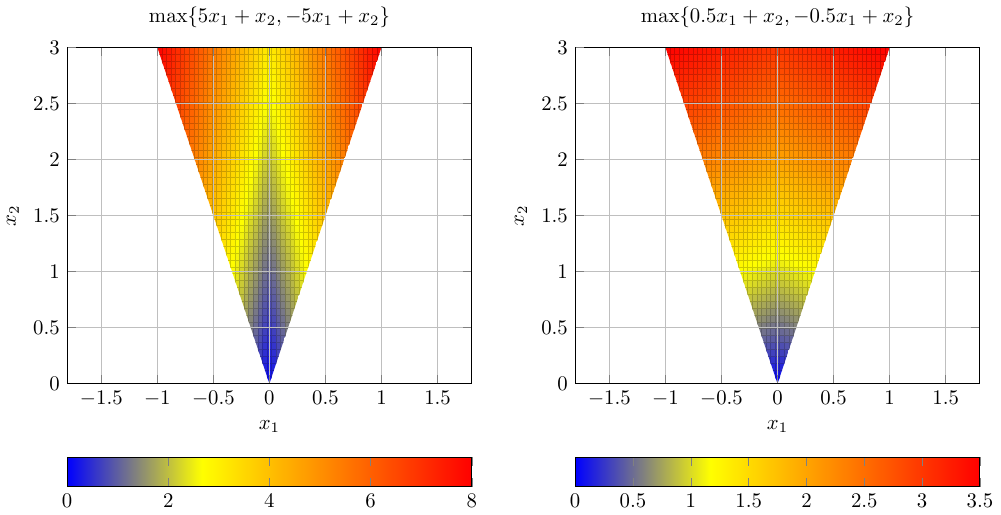}
\caption{Two convex examples of piecewise linear objective functions on a convex polytope. The FW algorithm fails on the left example but works on the right one. The objective functions of both examples have unbounded curvature constants. }
\label{fig:example}
\end{figure}

Note that when $x_1 = 0$, the function is not differentiable, and the subdifferential is given by $\text{conv}\{(5, 1), (-5, 1)\}$. In this case the FW algorithm chooses arbitrary subgradient from the subdifferential and computes the corresponding LMO. From the fundamental theorem of linear programming, it is easy to see that when $x_1 = 0$, $\text{LMO}(\mathbf{x})$ can be any of the three corners depending on the choice of subgradient at $\mathbf{x}$. Now suppose the FW algorithm stops in $T$ iterations. For any initial point $\mathbf{x}^{(0)} = (x_1^{(0)}, x_2^{(0)})$ where $x_1^{(0)} \neq 0$, the final point output by the FW algorithm will be a convex combination of $(-1, 3)$, $(1, 3)$, $(0, 0)$ and $\mathbf{x}^{(0)}$. Let $\{\gamma_t\}_{t=1}^T$ be the sequence of step sizes chosen by the FW algorithm. Then we can easily check that $x^{(T)}_2 \geq 3 - \prod_{t=1}^T(1-\gamma_t)(1 - x^{(0)}_2/3)\to 3$ as $T\to\infty$. Note that we can readily change this example by extending $\mathcal{D}$ so that the distance between $\mathbf{x}^{(T)}$ and the optimum $\mathbf{x}^*$ becomes arbitrarily large. Furthermore, both the line search and the fully-corrective variants fail since the vertices picked by the algorithm remains the same: $\{(-1, 3), (1, 3)\}$. Finally, for any regular probability distribution that is absolutely continuous w.r.t. the Lebesgue measure, with probability 1 the initial points sampled from the distribution will converge to suboptimal solutions. As a comparison, it can be shown that the subgradient method works for this problem since the function $f$ itself is Lipschitz continuous~\citep{nesterov2013introductory}.

Pick $\mathbf{x} = (-\eps, \delta)$, $\mathbf{s} - \mathbf{x} = (1, 0)$ and $\mathbf{y} = \mathbf{x} + \gamma(\mathbf{s} - \mathbf{x}) = (\gamma - \eps, \delta)$, where $\gamma > \eps$, and plug them in the definition of the curvature constant $C_f$. We have:
\begin{align*}
C_f &\geq \lim_{\eps\to 0^, \gamma > \eps}\frac{2}{\gamma^2}(f(\mathbf{y}) - f(\mathbf{x}) - \nabla f(\mathbf{x})^T(\mathbf{y} - \mathbf{x})) \\
&= \lim_{\gamma\to 0_+}\frac{20}{\gamma} = \infty
\end{align*}
i.e., the curvature constant of this piecewise linear function is unbounded. The problem for this failure case of FW lies in the fact that the curvature constant is infinity.

On the other hand, we can also show that FW works even when $C_f = \infty$ by slightly changing the objective function while keeping the constraint set:
\begin{align}
\underset{x_1, x_2}{\text{minimize}} & \quad\max\{\frac{1}{2}x_1 + x_2, -\frac{1}{2}x_1 + x_2\} \\
\text{subject to} & \quad x_2 \leq 3, 3x_1 + x_2 \geq 0, -3x_1 + x_2 \geq 0\nonumber
\label{equ:counterexample}
\end{align}
The objective function of the second example is shown in the right figure of Fig.~\ref{fig:example}. Still, the unique global optimum for the new problem is given by $\mathbf{x}^* = (0, 0)$ with $f(\mathbf{x}^*) = 0$, but now the $\text{LMO}(\mathbf{x})$ is:
$$\text{LMO}(\mathbf{x}) = (0, 0), \quad\forall \mathbf{x}\in\mathcal{D}$$
It is not hard to see that FW converges to the global optimum, and the curvature constant $C_f = \infty$ as well for this new problem. Combining with example (\ref{equ:example}), we can see that $C_f < \infty$ is not a necessary condition for the success of FW algorithm on convex problems, either. Piecewise linear functions form a rich class of objectives that are frequently encountered in practice, while depending on the structure of the problem, the FW algorithm may or may not work for them. This thus raises an interesting problem: \emph{can we develop a necessary condition for the success of the FW algorithm on convex problems?} Another interesting question is, \emph{can we develop sufficient conditions for piecewise linear functions under which the FW algorithm converges to global optimum?}

\section{NUMERICAL RESULTS}
We evaluate the effectiveness of the FW algorithm (Alg.~\ref{alg:fw}) in solving (\ref{equ:sof}) by comparing it with the penalty method~\citep{zhao2015sof} and the projected gradient descent method (PGD) on 4 datasets (Table~\ref{table:data}). These datasets are standard for clustering analysis: two of them are used in~\citep{zhao2015sof}, and we add two more datasets with various sizes to make a more comprehensive evaluation. The instances in each dataset are associated with true class labels. For each data set, we use the number of classes as the true number of clusters.
\begin{table}[htb]
\centering
\caption{Statistics about datasets.}
\label{table:data}
\begin{tabular}{l||ccc}\hline
Dataset & $\#$ inst. $(n)$ & $\#$ feats. $(p)$ & $\#$ clusters $(k)$\\\hline
blood & 748 & 4 & 10 \\
yeast & 1,484 & 8 & 10 \\
satimage & 4,435 & 36 & 6 \\
pendigits & 10,992 & 16 & 100\\\hline
\end{tabular}
\end{table}

\begin{figure*}[htb]
\centering
	\includegraphics[width=\linewidth]{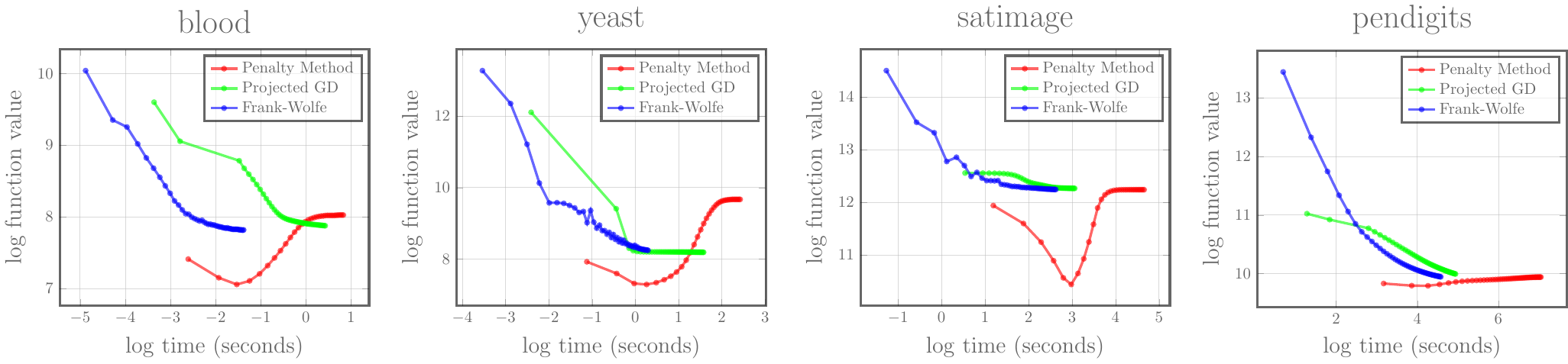}
\caption{Convergence speed of three algorithms. The $x$-axis measures $\log$-seconds and the $y$-axis measures $\log$ of objective function value. Note that the intermediate iterates of the penalty method are not feasible solutions, so we should only compare the convergent point of the penalty method with the other two. }
\label{fig:loss}
\end{figure*}
\begin{figure*}[htb]
\centering
	\includegraphics[width=\linewidth]{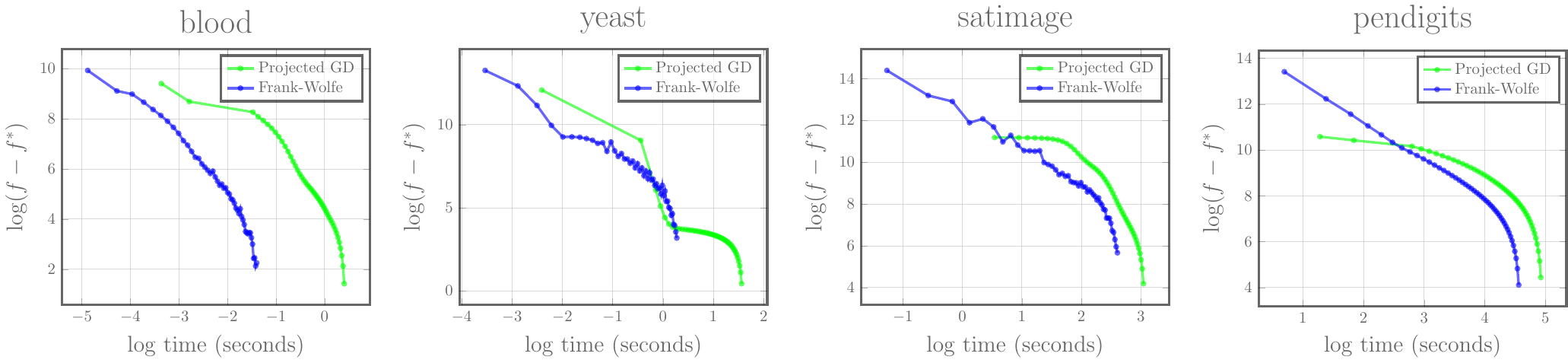}
\caption{Convergence speed comparison between the projected gradient descent method and the FW algorithm. The $x$-axis measures $\log$-seconds and the $y$-axis measures $\log(f - f^*)$, where $f$ is the objective function value and $f^*$ is the local optimum achieved by the algorithm. Note that the local optimum $f^*$ achieved by PGD and FW can be different. }
\label{fig:converge}
\end{figure*}

Given a data matrix $X\in\RR^{n\times p}$, we use a Gaussian kernel with fixed bandwidth 1.0 to construct the co-cluster matrix $P$ as $P_{ij} = \exp(-||\mathbf{x}_i - \mathbf{x}_j||_2^2)$. For each dataset, we use its number of clusters as the rank of the decomposition, i.e., $W \in\RR^{n\times k}$. We implement the penalty method based on~\citep{zhao2015sof}, where we set the maximum number of inner loops to be 50, and choose the step factor for the coefficients of the two penalty terms to be 2. In each iteration of PGD, we use backtracking line search to choose the step size. For all three algorithms, the stop conditions are specified as follows: if the difference of function values in two consecutive iterations is smaller than the fixed gap $\eps = 10^{-3}$, or the number of (outer) iterations exceed 50, then the algorithms will stop. Also, for each dataset, all three algorithms share the same initial point.

We plot the convergence speed to local optimum of these three algorithms in Fig.~\ref{fig:loss}. Clearly, the penalty method is orders of magnitude slower than both PGD and the FW algorithm, and it usually converges to a worse solution. On the other hand, although the FW algorithm tends to have more iterations before it converges, due to its cheap computation in each iteration, it consistently takes less time than PGD. Another distinction between PGD and FW is that PGD, when implemented with backtracking line search, is a monotone descent method, while FW is not. For a better visualization to compare between the PGD and the FW algorithms, we omit the penalty method in Fig.~\ref{fig:converge}, and draw the log-gap plot of the four datasets. We can confirm from Fig.~\ref{fig:converge} that both PGD and the FW algorithm have roughly the same order of convergence speed for solving (\ref{equ:sof}), which is consistent with the theoretical result proved in the previous section (Thm.~\ref{thm:nonconvex}). However, the FW algorithm often converges faster than the PGD method, in terms of the gap between the objective function value and local optimum.

\section{CONCLUSION}
We propose a FW algorithm to solve the SymNMF problem under a simplicial constraint. Compared with existing solutions, the proposed algorithm enjoys a non-asymptotic convergence guarantee to KKT points, is simple to implement, contains no hyperparameter to be tuned, and is also demonstrated to be much more efficient in practice. Theoretically, we establish a close connection of this problem to the famous completely positive matrix factorization by providing an equivalent geometric description. We also derive a tight bound on the curvature constant of this problem. As a side contribution, we give a pair of nonsmooth convex examples where the FW algorithm converges or fails to converge to its optimum. This result raises an interesting question w.r.t. the necessary condition of the success of the FW algorithm.

\subsubsection*{Acknowledgements}
HZ would like to thank Simon Lacoste-Julien for his insightful comments and helpful discussions. HZ and GG gratefully acknowledge support from ONR, award number N000141512365.

\bibliographystyle{abbrvnat}
\bibliography{reference}

\newpage
\onecolumn
\appendix
\section{MISSING PROOFS}
\curvature*
\begin{proof}
Let $\forall \mathbf{x}, \mathbf{s}\in \mathcal{D}, \gamma\in(0, 1],$ and $\mathbf{y} = \mathbf{x} + \gamma(\mathbf{s} - \mathbf{x})$. The smoothness of $f$ implies that $f$ is continuously differentiable, hence we have:
\begin{align*}
&\phantom{{}={}} \Bigl\lvert f(\mathbf{y}) - f(\mathbf{x}) - \nabla f(\mathbf{x})^T(\mathbf{y} - \mathbf{x})\Bigr\rvert \\
%&= \Bigl\lvert \int_0^1 \nabla f(\mathbf{x} + t(\mathbf{y} - \mathbf{x}))^T(\mathbf{y} - \mathbf{x})~dt - \nabla f(\mathbf{x})^T(\mathbf{y} - \mathbf{x})\Bigr\rvert \\
&= \Bigl\lvert \int_0^1 \left(\nabla f(\mathbf{x} + t(\mathbf{y} - \mathbf{x})) - \nabla f(\mathbf{x})\right)^T(\mathbf{y} - \mathbf{x})~dt \Bigr\rvert && \text{(Mean-value theorem)}\\
&\leq  \int_0^1 \Bigl\lvert\left(\nabla f(\mathbf{x} + t(\mathbf{y} - \mathbf{x})) - \nabla f(\mathbf{x})\right)^T(\mathbf{y} - \mathbf{x})\Bigr\rvert~dt  &&\text{(Triangle inequality)}\\
&\leq \int_0^1 ||\nabla f(\mathbf{x} + t(\mathbf{y} - \mathbf{x})) - \nabla f(\mathbf{x}) || \cdot ||\mathbf{y} - \mathbf{x}||~dt  &&\text{(Cauchy-Schwarz inequality)}\\
&\leq \int_0^1 tL\gamma^2 ||\mathbf{s} - \mathbf{x}||^2~dt \leq \frac{L\gamma^2}{2}\text{diam}^2(\mathcal{D}) && \text{(Smoothness assumption of $f$)}
\end{align*}
It immediately follows that
$$\bar{C}_f \leq \frac{2}{\gamma^2}\frac{L\gamma^2}{2}\text{diam}^2(\mathcal{D}) = \text{diam}^2(\mathcal{D})L$$
\end{proof}

\lacoste*
\begin{proof}
Let $\mathbf{y}\defeq \mathbf{x} + \gamma \mathbf{d}$, where $\mathbf{d}\defeq \mathbf{s} - \mathbf{x}$ is the update direction found by the LMO in Alg.~\ref{alg:fw}. Using the definition of $\bar{C}_f$, we have:
\begin{align*}
f(\mathbf{y}) &= f(\mathbf{y}) - f(\mathbf{x}) - \gamma \nabla f(\mathbf{x})^T\mathbf{d} + f(\mathbf{x}) + \gamma \nabla f(\mathbf{x})^T\mathbf{d}\\
&\leq f(\mathbf{x})  + \gamma \nabla f(\mathbf{x})^T\mathbf{d} + \left| f(\mathbf{y}) - f(\mathbf{x}) - \gamma \nabla f(\mathbf{x})^T\mathbf{d}\right|\\
&\leq f(\mathbf{x})  + \gamma \nabla f(\mathbf{x})^T\mathbf{d} + \frac{\gamma^2}{2}\bar{C}_f
\end{align*}
Now using the definition of the FW gap $g(\mathbf{x})$ and for $\forall C \geq \bar{C}_f$, we get:
\begin{equation}
f(\mathbf{y}) \leq f(\mathbf{x}) - \gamma g(\mathbf{x}) + \frac{\gamma^2}{2}\bar{C}_f, \quad\forall \gamma \in (0, 1]
\label{equ:upper}
\end{equation}
Depending on whether $C > 0$ or $C= 0$, the R.H.S. of (\ref{equ:upper}) is a either a quadratic function with positive second order coefficient or an affine function. In the first case, the optimal $\gamma^*$ that minimizes the R.H.S. is $\gamma^* = g(\mathbf{x})/ C$. In the second case, $\gamma^* = 1$. Combining the constraint that $\gamma^* \leq 1$, we have $\gamma^* = \min\{1, g(\mathbf{x})/C\}$. Thus we obtain:
\begin{equation}
f(\mathbf{y}) \leq f(\mathbf{x}) - \min\left\{\frac{g^2(\mathbf{x})}{2C}, \left(g(\mathbf{x}) - \frac{C}{2}\right)\mathbb{I}_{g(\mathbf{x}) > C}\right\}
\label{equ:recur}
\end{equation}
(\ref{equ:recur}) holds for each iteration in Alg.~\ref{alg:fw}. A cascading sum of (\ref{equ:recur}) through iteration step $1$ to $T+1$ shows that:
\begin{equation}
f(\mathbf{x}^{(T+1)})\leq f(\mathbf{x}^{(0)}) - \sum_{t=0}^{T}\min\left\{\frac{g^2(\mathbf{x}^{(t)})}{2C}, \left(g(\mathbf{x}^{(t)}) - \frac{C}{2}\right)\mathbb{I}_{g(\mathbf{x}^{(t)}) > C}\right\}
\label{equ:cascade}
\end{equation}
Define $\tilde{g}_T\defeq \min_{0\leq t\leq T}g(\mathbf{x}^{(t)})$ be the minimal FW gap in $T+1$ iterations. Then we can further bound inequality (\ref{equ:cascade}) as:
\begin{equation}
f(\mathbf{x}^{(T+1)})\leq f(\mathbf{x}^{(0)}) - (T+1)\min\left\{\frac{\tilde{g}_T^2}{2C}, \left(\tilde{g}_T - \frac{C}{2}\right)\mathbb{I}_{\tilde{g}_T > C}\right\}
\label{equ:key}
\end{equation}
We discuss two subcases depending on whether $\tilde{g}_T > C$ or not. The main idea is to get an upper bound on $\tilde{g}_T$ by showing that $\tilde{g}_T$ cannot be too large, otherwise the R.H.S. of (\ref{equ:key}) can be smaller than the global minimum of $f$, which is a contradiction. For the ease of notation, define $h_0\defeq f(\mathbf{x}^{(0)}) - \min_{\mathbf{x}\in\mathcal{D}}f(\mathbf{x})$, i.e., the initial gap to the global minimum of $f$.

\textbf{Case I}. If $\tilde{g}_T > C$ and $\tilde{g}_T - \frac{C}{2}\leq \frac{\tilde{g}_T^2}{2C}$, from (\ref{equ:key}), then:
$$0 \leq f(\mathbf{x}^{(T+1)}) - \min_{\mathbf{x}\in\mathcal{D}}f(\mathbf{x})\leq f(\mathbf{x}^{(0)}) - \min_{\mathbf{x}\in\mathcal{D}}f(\mathbf{x}) - (T+1)(\tilde{g}_T - \frac{C}{2}) = h_0 - (T+1)(\tilde{g}_T - \frac{C}{2})$$
which implies
$$C < \tilde{g}_T \leq \frac{h_0}{T+1} + \frac{C}{2} \Rightarrow \tilde{g}_T\leq \frac{2h_0C}{T+1} = O(1/T)$$
On the other hand, solving the following inequality:
$$C -\frac{C}{2}\leq \tilde{g}_T - \frac{C}{2} \leq \frac{\tilde{g}^2_T}{2C} \leq \frac{4h_0^2C^2}{(T+1)^2}\frac{1}{2C}$$
we get
$$T + 1 \leq 2h_0$$
This means that $\tilde{g}_T$ decreases in rate $O(1/T)$ only for at most the first $2h_0$ iterations.

\textbf{Case II}. If $\tilde{g}_T \leq C$ or $\tilde{g}_T - \frac{C}{2} > \frac{\tilde{g}_T^2}{2C}$. Similarly, from (\ref{equ:key}), we have:
$$0 \leq f(\mathbf{x}^{(T+1)}) - \min_{\mathbf{x}\in\mathcal{D}}f(\mathbf{x})\leq f(\mathbf{x}^{(0)}) - \min_{\mathbf{x}\in\mathcal{D}}f(\mathbf{x}) - (T+1)\frac{\tilde{g}^2_T}{2C} = h_0 - (T+1)\frac{\tilde{g}^2_T}{2C}$$
which yields
$$\tilde{g}_T \leq \sqrt{\frac{2h_0 C}{T+1}}$$

Combining the two cases together, we get $\tilde{g}_T\leq \frac{2h_0 C}{T+1}$ if $T + 1 \leq 2h_0$; otherwise $\tilde{g}_T \leq \sqrt{\frac{2h_0C}{T+1}}$. Note that for $T \geq 0$, $\sqrt{T+1}\leq T+1$, thus we can further simplify the upper bound of $\tilde{g}_T$ as:
$$\tilde{g}_T \leq \frac{\max\{2h_0C, \sqrt{2h_0 C}\}}{\sqrt{T+1}}$$
\end{proof}

\hessian*
\begin{proof}
Using the theory of matrix differential calculus, the Hessian of a matrix-valued matrix function is defined as:
$$\nabla^2 f(W) \defeq \frac{\partial \vect{\nabla f(W)}}{\partial \vect{W}}$$
Using the differential notation, we can compute the differential of $\nabla f(W)$ as:
$$
\d\nabla f(W) = \d (WW^T - P)W = (\d W)W^TW + W(\d W)^TW + WW^T\d W - P\d W
$$
Vectorize both sides of the above equation and make use of the identity that $\vect(ABC) = (C^T\otimes A)\vect{B}$ for $A, B, C$ with appropriate shapes, we get:
$$\vect{\d~\nabla f(W)} = (W^TW\otimes I_n)\vect{\d W} + (W^T\otimes W)\vect{\d W^T} + (I_k\otimes (WW^T - P))\vect{\d W}$$
Let $K_{nk}$ be a commutation matrix such that $K_{nk}\vect{W} = \vect{W^T}$. We can further simplify the above equation as:
\begin{equation}
\vect{\d\nabla f(W)} = \left( W^TW\otimes I_n + (W^T\otimes W)K_{nk} + I_k\otimes (WW^T - P) \right)\vect{\d W}
\end{equation}
It then follows from the first identification theorem~\citep[Thm. 6]{magnus1985matrix} that the Hessian is given by
$$\nabla^2 f(W) = \left( W^TW\otimes I_n + I_k\otimes (WW^T - P) + (W^T\otimes W)K_{nk}\right)\in\RR^{nk\times nk}$$
As a sanity check, the first two terms in $\nabla^2 f(W)$ are clearly symmetric. The third term can be verified as symmetric as well by realizing that $K^{-1}_{nk} = K^T_{nk}$, and
$$W\otimes W^T = K_{nk}(W^T\otimes W)K_{nk}$$
\end{proof}

\eigenbound*
\begin{proof}
$\forall W\geq 0$, if $W\mathbf{1}_k = \mathbf{1}_n$, then by the Courant-Fischer theorem:
\begin{align*}
||W^TW||_2 & \defeq \max_{\substack{\mathbf{v}\in\RR^k, \\ ||\mathbf{v}||_2 = 1}}||W^TW\mathbf{v}||_2 && \text{(Courant-Fischer theorem)}\\
&= \max_{\substack{\mathbf{v}\in\RR_+^k, \\ ||\mathbf{v}||_2 = 1}}||W^TW\mathbf{v}||_2 && \text{(Perron-Frobenius theorem)}\\
& \leq\max_{\substack{\mathbf{v}\in\RR_+^k, \\ ||\mathbf{v}||_\infty \leq 1}}||W^TW\mathbf{v}||_2 && \text{($B_2(0, 1)\subseteq B_\infty(0, 1)$)}\\
&= ||W^T\mathbf{1}_n||_2 && (W \geq 0, W\mathbf{1}_k = \mathbf{1}_n) \\
&\leq ||W^T\mathbf{1}_n||_1 = n
\end{align*}
To achieve this upper bound, consider $W = \mathbf{1}_n e_1^T$, where $e_1$ is the first column vector of the identity matrix $I_k$. In this case $W^TW = e_1 \mathbf{1}_n^T\mathbf{1}_n e_1^T = ne_1e_1^T$, which is a rank one matrix with a positive eigenvalue $n$. Hence $\sup ||W^TW||_2 = n$.
\end{proof}

\smooth*
\begin{proof}
Recall that the spectral norm $||\cdot||_2$ is sub-multiplicative and the spectrum of $A\otimes B$ is the product of the spectrums of $A$ and $B$. Using (\ref{equ:hessian}), we have:
\begin{align*}
||\nabla^2 f(W)||_2 &=  ||W^TW\otimes I_n + I_k\otimes (WW^T - P) + (W^T\otimes W)K_{nk}||_2 \\
&\leq ||W^TW\otimes I_n||_2 + ||I_k\otimes (WW^T - P)||_2 + ||(W^T\otimes W)K_{nk}||_2 && \text{(Triangle inequality)}\\
&= ||W^TW||_2|| I_n||_2 + ||I_k||_2 ||WW^T - P||_2 + ||W^T\otimes W||_2 ||K_{nk}||_2 && \text{(submultiplicativity of $||\cdot||_2$)}\\
&= ||W^TW||_2 + ||WW^T - P||_2 + ||W^T\otimes W||_2 &&\text{($||I_n||_2 = ||I_k||_2 = ||K_{nk}||_2 = 1$)} \\
&\leq 3||W^TW||_2 + ||P||_2 && \text{(Triangle inequality)}\\
&\leq 3n + c && \text{(Lemma~\ref{thm:eigenbound})}
\end{align*}
The result then follows from Lemma~\ref{lemma:lipschitz}.
\end{proof}

\diameter*
\begin{proof}
\begin{align*}
\text{diam}^2(\mathcal{D}) &= \sup_{W, Z\in\mathcal{D}}||W - Z||_F^2 \\
&= \sup_{W, Z\in\mathcal{D}}\sum_{ij}(W_{ij} - Z_{ij})^2 = \sup_{W, Z\in\mathcal{D}}\sum_{ij}W^2_{ij} + Z^2_{ij} - 2W_{ij}Z_{ij} \\
&\leq \sup_{W, Z\in\mathcal{D}}\sum_{W, Z\in\mathcal{D}} W^2_{ij} + Z^2_{ij} \leq \sup_{W, Z\in\mathcal{D}}\sum_{W, Z\in\mathcal{D}} W_{ij} + Z_{ij} \\
&= 2n
\end{align*}
Note that choosing $W = \mathbf{1}e_1^T$ and $Z = \mathbf{1}_n e_2^T$ make all the equalities hold in the above inequalities. Hence $\text{diam}^2(\mathcal{D}) = 2n$.
\end{proof}

\bound*
\begin{proof}
For a matrix $A$, we will use $\sigma_i(A)$ to mean the $i$th largest singular value of $A$ and $\lambda_{max}(A)$, $\lambda_{min}(A)$ to mean the largest and smallest eigenvalues of $A$, respectively. Recall $\nabla^2 f(W) = W^TW\otimes I_n + I_k\otimes (WW^T - P) + (W^T\otimes W)K_{nk}$. For $W\geq 0, W\mathbf{1}_k = \mathbf{1}_n$, let $r = \text{rank}(W)$. Clearly $r\geq 1$. We have the following inequalities hold:
\begin{align*}
||\nabla^2 f(W)||_2 &= ||W^TW\otimes I_n + I_k\otimes (WW^T - P) + (W^T\otimes W)K_{nk}||_2\\
&\geq \lambda_{max}\left(WW^T\otimes I_n + (W^T\otimes W)K_{nk}\right) + \lambda_{min}\left(I_k\otimes (WW^T- P)\right) && \text{(Weyl's inequality)}\\
&\geq \lambda_{max}(WW^T\otimes I_n) + \lambda_{min}\left((W^T\otimes W)K_{nk}\right) + \lambda_{min}\left(I_k\otimes (WW^T - P)\right) \\
&= \lambda_{max}(WW^T) + \lambda_{min}(W^T\otimes W) + \lambda_{min}(WW^T - P)  \\
&\geq \lambda_{max}(WW^T) + \lambda_{min}(W^T\otimes W) + \lambda_{min}(WW^T) - \lambda_{max}(P) \\
&= \sigma^2_1(W) + 2\sigma^2_r(W) - \lambda_{max}(P) \\
&\geq \sigma^2_1(W) - c  && (||P||_2 \leq ||P||_F) \\
& \geq \frac{1}{r} ||W||_F^2 - c  && (r\cdot\sigma^2_1(W)\geq ||W||_F^2)\\
& \geq \frac{1}{k}||W||_F^2 - c && (\text{rank}(W)\leq k)\\
&= \frac{1}{k}\sum_{i=1}^n\sum_{j=1}^k W_{ij}^2 - c \\
&\geq \frac{1}{k}\sum_{i=1}^n k\left(\frac{\sum_{j=1}^k W_{ij}}{k}\right)^2 - c && \text{(Cauchy ineq.)}\\
&= \frac{n}{k^2} - c
\end{align*}
where the first three inequalities all follow from Weyl's inequality.
\end{proof}

\end{document}